\documentclass[12pt]{article}

\usepackage[english]{babel}
\usepackage[utf8x]{inputenc}
\usepackage[T1]{fontenc}

\usepackage[letterpaper,top=1in,bottom=1in,left=1in,right=1in,marginparwidth=1.75cm]{geometry}

\usepackage{amsmath,amsfonts,amsthm,amssymb,mathrsfs,dsfont} 
\usepackage{mathtools}
\usepackage{graphicx}
\usepackage[colorinlistoftodos]{todonotes}
\usepackage[colorlinks=true, allcolors=blue]{hyperref}
\usepackage[capitalize,nameinlink]{cleveref}
\usepackage{verbatim}

\newcommand{\bE}{\mathbb{E}}
\newcommand{\R}{\mathbb{R}}
\newcommand{\F}{\mathcal{F}}
\newcommand{\I}{\mathcal{I}}
\newcommand{\eps}{\varepsilon}

\newcommand{\E}{\bE}      

\newcommand{\KL}{\mathop{\bf KL\/}}
\newcommand{\bone}{\boldsymbol{1}}

\newcommand{\vnote}[1]{\textcolor{red}{\small {\textbf{(Vishesh: }#1\textbf{) }}}}
\newcommand{\fnote}[1]{\textcolor{blue}{\small {\textbf{(Fred: }#1\textbf{) }}}}

\newtheorem{theorem}{Theorem}[section]
\newtheorem*{namedtheorem}{\theoremname}
\newcommand{\theoremname}{testing}

\newtheorem{lemma}[theorem]{Lemma}

\newtheorem{proposition}[theorem]{Proposition}

\newtheorem*{question*}{Question}

\theoremstyle{definition}

\newtheorem{defn}[theorem]{Definition}
\newtheorem{remark}[theorem]{Remark}
\newtheorem{example}[theorem]{Example}

\theoremstyle{plain}

\usepackage{algpseudocode,algorithm,algorithmicx}

\title{The Vertex Sample Complexity of Free Energy is Polynomial}
\author{Vishesh Jain\thanks{Massachusetts Institute of Technology. Department of Mathematics. Email: {\tt visheshj@mit.edu}} \and Frederic Koehler\thanks{Massachusetts Institute of Technology. Department of Mathematics. Email: {\tt fkoehler@mit.edu}} \and Elchanan Mossel\thanks{Massachusetts Institute of Technology. Department of Mathematics and IDSS. Supported by ONR grant N00014-16-1-2227   and 
NSF CCF-1665252 and DMS-1737944. Email: {\tt elmos@mit.edu} } }

\date{}

\begin{document}
\maketitle
\thispagestyle{empty}
\setcounter{page}{0}

\begin{abstract}
The free energy is a key quantity which is associated to Markov random fields.
Classical results in statistical physics show how, given an analytic formula of the free energy, it is possible to compute many key quantities associated with Markov random fields including quantities such as magnetization and the location of various phase transitions. 
Given a massive Markov random field on $n$ nodes, can a small sample from it provide a rough approximation to the free energy $\F_n = \log{Z_n}$? 

Results in graph limit literature by Borgs, Chayes, Lov{\'a}sz, S{\'o}s, and Vesztergombi show that 
for Ising models on $n$ nodes and interactions of strength $\Theta(1/n)$, an $\epsilon$ approximation to $\log Z_n / n$ can be achieved by sampling a randomly induced model on $2^{O(1/\epsilon^2)}$ nodes. We show that the sampling complexity of this problem is {\em polynomial in }$1/\eps$. We further show a polynomial dependence on $\epsilon$ cannot be avoided. 

Our results are very general as they apply to higher order Markov random fields. For Markov random fields of order $r$, we obtain an algorithm that achieves $\epsilon$ approximation using a number of samples polynomial in $r$ and $1/\epsilon$ and running time that is $2^{O(1/\epsilon^2)}$ up to polynomial factors in $r$ and $\epsilon$. For ferromagnetic Ising models, the running time is polynomial in $1/\epsilon$. 

Our results are intimately connected to recent research on the regularity lemma and property testing, where the interest is in finding which properties can tested within $\epsilon$ error in time polynomial in $1/\epsilon$. In particular, our proofs build on results from a recent work by Alon, de la Vega, Kannan and Karpinski, who also introduced the notion of polynomial vertex sample complexity. Another critical ingredient of the proof is an effective bound by the authors of the paper relating the variational free energy and the free energy. \end{abstract}

\newpage

\section{Introduction}
One of the major problems in the areas of Markov Chain Monte Carlo (MCMC), statistical inference, and machine learning is approximating the partition function of Ising models (and more generally, Markov random fields). An \emph{Ising model} is specified by
a probability distribution on the discrete cube $\{\pm1\}^n$ of the form
\[ P[X = x] := \frac{1}{Z} \exp(\sum_{i,j} J_{i,j} x_i x_j) = \frac{1}{Z} \exp(x^T J x), \]
where the collection $\{J_{i,j}\}_{i,j\in\{1,\dots,n\}}$ are the entries of
an arbitrary real, symmetric matrix with zeros on the diagonal. The distribution $P$ is referred to as the \emph{Boltzmann distribution}. The normalizing constant $Z=\sum_{x\in\{\pm1\}^{n}}\exp(\sum_{i,j=1}^{n}J_{i,j}x_{i}x_{j})$
is called the \emph{partition function }of the Ising model and the quantity $\F := \log{Z}$ is known as the \emph{free energy}. 

The free energy is a key physical quantity which has long been studied in statistical physics due to the wealth of information it reveals about the underlying Ising model. Some textbook applications of the analysis of the free energy include the computation of fundamental quantities like the net magnetization (this is discussed in detail in \cref{appendix-magnetization-proof}), and the location of \emph{phase transitions} in parameterized families of Ising models. We refer the reader to \cite{ellis2007entropy} for much more on this. In recent years, the study of the free energy has also proved to be very fruitful in non-physical applications of the Ising model. For instance, consider the problem in combinatorial optimization of maximizing the quadratic form $x \mapsto x^T M x$ over the hypercube $\{\pm 1\}^{n}$; this is essentially the problem of estimating the cut norm of a matrix and has max-cut as the special case when all of the entries are negative. The free energy of the model with interaction matrix $J_{\beta}:=\beta M$ provides a natural tempering of this optimization problem in the following sense:   
\[ \frac{1}{\beta} \mathcal{F}_{\beta} = \frac{1}{\beta} \log \sum_{x\in \{\pm 1\}^{n}}\exp\left(\beta \sum_{i,j=1}^{n}M_{ij}x_{i}x_{j}\right) \to \max_{x \in \{\pm 1\}^n} \sum_{i,j = 1}^n M_{ij} x_i x_j \]
as $\beta \to \infty$.

In fact for every finite $\beta$, the free energy corresponds
to the objective value of a natural optimization problem of its own.
More precisely the free energy is characterized by the following \emph{variational principle} (dating back to Gibbs, see the references in \cite{ellis2007entropy}):
\begin{equation}
\label{eqn:free-energy-variational-char}
\F = \max_{\mu} \left[\sum_{i,j} J_{ij} \E_{\mu}[X_i X_j] + H(\mu)\right],
\end{equation}
where $\mu$ ranges over all probability distributions on the boolean hypercube $\{\pm 1\}^{n}$. This can be seen by noting that 
\begin{equation}
\label{eqn:free-energy-KL}
\KL(\mu ||P)=\F - \sum_{i,j} J_{ij} \E_{\mu}[X_i X_j] - H(\mu),
\end{equation}
and recalling that $\KL(\mu ||P) \geq 0$ with equality if and only if $\mu = P$.
 
By substituting $J = \beta M$ in equation \cref{eqn:free-energy-variational-char}, we  see 
that the Boltzmann distribution is simply the maximum entropy distribution $\mu$ for a fixed value
of the expected energy $\E_{\mu}[x^T M x]$. Thus, studying the free
energy for different values of $\beta$ provides much richer information about
the optimization landscape of $x \mapsto x^T M x$ over the hypercube than just the maximum value, e.g. in the max-cut case, the free energies encode
information about non-maximal cuts as well (see e.g. \cite{borgs2012convergent} for related discussion).

Apart from the applications mentioned above, it is clear by definition that knowledge of the free energy (or equivalently, the partition function) allows one to perform fundamental inference tasks like computing marginals and posteriors in Ising models. Unfortunately, the partition function, which is defined as a sum of exponentially many terms, turns out to be both theoretically and computationally intractable. For instance, it is known that approximating the partition function is NP-hard, even for graphs with degrees bounded by a small constant (see \cite{sly-sun}), whereas providing a closed form expression for the partition function of the Ising model on the standard $3$-dimensional lattice remains one of the outstanding problems in statistical physics. In light of this, providing efficient approximation schemes for the free energy, which have provable guarantees, has naturally attracted considerable interest over the years. 

The work of Jerrum and Sinclair~\cite{JerrumSinclair:89b} showed that it is possible to approximate the partition function for 
``self-reducible'' models for which a rapidly mixing Markov chain exists. Moreover, for such models, a 
$(1+\epsilon)$ approximation of the partition function results in a rapidly mixing chain. 
Some key results in the theory of MCMC provide conditions for the existence of a rapidly 
mixing chain, and therefore allow for efficient approximations of the partition functions e.g.~\cite{JerrumSinclair:89,JerrumSinclair:90,JeSiVi:04} and follow up work. 

On the other hand, even in interesting regimes where correlation decay does not hold (and therefore, MCMC techniques do not provide non-trivial guarantees), much less is known. In \cite{risteski-ising}, Risteski used variational methods (based on \cref{eqn:free-energy-variational-char}) and convex programming hierarchies to provide an $O(\epsilon n)$-additive approximation to the free energy of suitably dense Ising models in time $n^{O(1/\epsilon^2)}$. In \cite{previous-paper}, the authors of this paper provided an algorithm with similar guarantees which works under weaker density assumptions, and additionally, runs in \emph{constant time} $2^{\tilde{O}(1/\epsilon^2)}$. We note that both Risteski's algorithm and the algorithm in \cite{previous-paper} generalise to order $r$ Markov random fields (MRFs) -- for fixed $r$, his algorithm provides an $O(\epsilon n)$- additive approximation to the free energy of sufficiently dense MRFs in time $n^{O(1/\epsilon^2)}$, whereas our algorithm provided a similar guarantee under weaker density assumptions either in time $n^{r}2^{\tilde{O}(1/\epsilon^{2})}$, or in constant time $2^{\tilde{O}(1/\epsilon^{2r-2})}$. As one of the applications of our main result, we will improve this running time guarantee to $2^{\tilde{O}(1/\epsilon^{2})}$ for all order $r$ MRFs. 

\begin{remark}
We note a recent preprint by the authors titled ``Approximating Partition Functions in Constant Time'' \cite{old-paper}.
\cite{old-paper} is completely superseded by the current work and \cite{previous-paper}. 
The current work builds on the main result of \cite{previous-paper} which provides an effective bound 
on the difference between the free energy and the variational free energy. 
 \end{remark}

\subsection{The vertex sample complexity: main results}
Most relevant to our paper is the work of Alon, de la Vega, Kannan and Karpinski \cite{alon-etal-samplingCSP}, who provided the following scheme for approximating MAX-CUT to additive error $\epsilon n^2$ for any $\epsilon > 0$: sample a random subset of vertices of size $q$, solve MAX-CUT on the  graph induced on the sampled vertices, and rescale this value by $n^{2}/q^{2}$. They defined the \emph{vertex sample complexity} to be the value of $q$ needed to achieve such an approximation (say, with probability $0.9$). Their key result showed that $q$ can be taken to be polynomial in $\epsilon^{-1}$. Moreover, they obtained a similar result for general MAX-rCSPs with vertex sample complexity $q = C_r poly(1/\epsilon)$, where we emphasize that the only way $q$ depends on $r$ is through the constant $C_r$. We refer the reader to the discussion in \cite{Alon:06} for an overview of similar results. 

Vertex sample complexity is also one of the central parameters of interest in graph property testing, where it is more commonly known as \emph{query complexity}. Roughly speaking, in the area of graph property testing initiated by Goldreich, Goldwasser and Ron \cite{goldreich1998property}, the goal is to efficiently test when a given graph satisfies some property $\Pi$ (defined to be a set of graphs closed under graph isomorphisms) versus when it is `sufficiently far' from satisfying this property, by selecting a small number of vertices at random and inspecting the graph induced on these sampled vertices. For instance, a model result in graph property testing would give an upper bound on the number of vertices $q = q(\epsilon)$ that one needs to sample in order to say with high probability that either a given graph is triangle-free, or that one needs to remove at least $\epsilon n^2$ edges from it to make it triangle-free. The question of which graph properties have query complexity $q=q(\epsilon)$ independent of the size of the graph was the focus of considerable effort by many researchers, culminating in the work of Alon and Shapira \cite{alon2008characterization}, who provided a characterization of `natural' graph properties which are testable with one-sided error. However, their proof relied on the so-called strong regularity lemma, and gave Ackermann type bounds. In recent years, there has been much work (see, e.g. \cite{gishboliner2016removal}, \cite{gishboliner2017efficient} and the references therein) to determine which graph properties are testable with a number of queries which is polynomial in $\epsilon^{-1}$.     

Our main result is that the vertex sample complexity of free energy is polynomial. Fix an Ising model $J$ on the vertex set $[n]$, and denote its free energy by $\F$. Consider a random subset $Q$ of $[n]$ of size $|Q|=q$. Consider also the Ising model $J_Q$ on the vertex set $Q$ whose matrix of interaction strengths is given by the restriction of the matrix $\frac{n}{q}J$ to $Q\times Q$. We will denote the free energy of this Ising model by $\F_Q$. 
\begin{theorem}
\label{thm:sample-complexity-free-energy}
Let $\epsilon > 0$ and suppose $q \ge 128000\omega$, where $\omega:=\log(1/\epsilon)/\epsilon^{8}$. 
Then, with probability at least $19/20$:
$$\left|\F - \frac{n}{q}\F_Q\right| \leq 4000\epsilon n \left(\|J\|_F + \epsilon n \|\vec{J}\|_{\infty} + \omega/q \right).$$
\end{theorem}
Here, $\|J\|_{F}:= \sqrt{\sum_{i,j}J_{i,j}^2}$ denotes the Frobenius norm of the matrix $J$ and $\|\vec{J}\|_{\infty}$ denotes the absolute value of its largest entry. Note that we assume that $\omega/q \le 1/128000$, so that the last term is almost always negligible. 

This result is \emph{tight up to the power of $\epsilon$ in $\omega$}. More precisely, we show the following lower bound:
\begin{theorem}
\label{thm:sample-complexity-lower-bound}
Let $\epsilon > 0$ and suppose $q \le 1/\sqrt{60000\epsilon}$. Then, there exists an Ising model $J$ for which, 
with probability at least $1/4$:
$$\left|\F - \frac{n}{q}\F_Q\right| > 4000\epsilon n \left(\|J\|_F + \epsilon n \|\vec{J}\|_{\infty} + 1 \right).$$
\end{theorem}

Our methods extend in a straightforward manner not just to Ising models with external fields, but indeed to general higher order Markov random fields, as long
as we assume a bound $r$ on the order of the highest interaction (i.e. size of the largest hyper-edge). 

\begin{defn}
Let $J$ be an arbitrary function on the hypercube $\{ \pm 1\}^n$
and suppose that the degree of $J$ is $r$ i.e. the Fourier decomposition of $J$ is $J(x) = \sum_{\alpha \subset [n]} J_{\alpha} x^{\alpha}$ with $r = \max_{J_{\alpha} \ne 0} |\alpha|$.
The corresponding \emph{order $r$ (binary) Markov random field} is the probability distribution
on $\{\pm 1\}^n$ given by
\[ P(X = x) = \frac{1}{Z}\exp(J(x)) \]
where the normalizing constant $Z$ is referred to as the \emph{partition function}.
For any polynomial $J$ we define $J_{=d}$ to be its $d$-homogeneous part and
 $\|J\|_F$ to be the square root of the total Fourier energy of $J$ i.e. $\|J\|_F^2 := \sum_{\alpha} |J_{\alpha}|^2$.
\end{defn}

Exactly as for Ising models, we can also define the free energy (which we continue to denote by $\F = \log Z$) for order $r$ Markov random fields. The analogous definition of $\F_Q$ is the free energy corresponding to the restriction of the polynomial $\tilde{J} := \sum_{\alpha \subseteq [n]}\frac{n^{|\alpha|-1}}{q^{|\alpha|-1}}J_\alpha x^\alpha $ to $\{\pm 1\}^Q$.

\begin{theorem}
\label{thm-mrf-sample-complexity}
Fix $J$ an order $r$ Markov random field.
Let $\epsilon > 0$ and suppose $q \ge 10^6\omega$, where $\omega:= r^7\log(1/\epsilon)/\epsilon^{8}$. 
Then, with probability at least $39/40$:
$$\left|\F - \frac{n}{q}\F_Q\right| \leq 10^5\epsilon r^3 \sum_{d = 1}^r n^{d/2} \left(\|J_{=d}\|_F + \epsilon n^{d/2} \|\vec{J}\|_{\infty} + \omega/q\right).$$
\end{theorem}

\subsection{Examples}
We discuss a few examples of natural families of Ising models and Markov random fields in order to illustrate the consequences of our results.

This example will illustrate that the exact size of the sample we want to take may depend on the density of the graph: with
the natural scalings from \cref{example:uniform-edge-weights} we see that for very sparse graphs this approach will not give good results, because if we take small samples we will just get the empty graph. On the other hand if the graph has average degree $\Theta(n)$, we will be able to approximate the free energy density $\F/n$ to $\epsilon$ additive error using samples which are of constant size $poly(1/\epsilon)$ without any dependence on $n$. To do the same for graphs with average degree $o(n)$, our sample size will need to grow with $n$ but depending on the precise level of sparsity we may still be able to take samples which are much smaller than the original graph.
\begin{example}[Uniform edge weights on graphs of increasing degree]\label{example:uniform-edge-weights}
Fix $\beta \in \mathbb{R}$ and a sequence of graphs
$(G_{n_i})_{i = 1}^{\infty}$ with the number of vertices $n_i$ going to infinity, and let $m_i$ be the corresponding number of edges. Then, it is natural to look at the model
with uniform edge weights equal to $\beta n_i/m_i$, since this makes
the maximum value of $x^T J x$ on the order of $\Theta(n_i)$, which is 
the same scale as the entropy term in the variational definition of the free energy (\cref{eqn:free-energy-variational-char}). We say the model is \emph{ferromagnetic} if $\beta > 0$ and \emph{anti-ferromagnetic} if $\beta < 0$. Observe that $\|J\|_F = |\beta| n_i/\sqrt{m_i}$ and $\|\vec{J}\|_{\infty} = |\beta|n_i/m_i$, so that by \cref{thm:sample-complexity-free-energy}, we have $|\mathcal{F}/n_i - \mathcal{F}_Q/q_i| = O(\epsilon (n_i/\sqrt{m_i} + \epsilon n_i^2/m_i + \omega/q))$. Suppose
$m_i = \Theta(n_i^{2(1 - \delta)})$, then this simplifies to
$|\mathcal{F}/n_i - \mathcal{F}_Q/q_i| = O(\epsilon (n_i^{\delta} + \epsilon n_i^{2\delta} + \omega/q))$. Finally, taking $\epsilon = \Theta(n_i^{-\delta})$, we see that with sample size $q = \Theta(n_i^{8 \delta} \log n_i)$, we can get $|\mathcal{F}/n_i - \mathcal{F}_Q/q_i|$ arbitrarily small.
\end{example}

\begin{example}[Uniform edge weights on $r$-uniform hypergraphs]
Fix $\beta \in \mathbb{R}$ and
let $(G_{n_i})_{i=1}^{\infty}$ be a sequence of $r$-uniform hypergraphs
with $n_i$ vertices and $m_i$ hyperedges. Analogous to the graph case, we let $J(x) = \frac{\beta n_i}{m_i} \sum_{S \in E(G_{n_i})} x_S$,
so that the maximum of $J$ is on the same order as the entropy term in the free energy. We still have $\|J\|_F = \beta n_i/\sqrt{m_i}$, and see by \cref{thm-mrf-sample-complexity} that $|\mathcal{F}/n_i - \mathcal{F}_Q/q_i| = O(\epsilon (n_i^{r/2}\log{n_i}/m_i^{1/2} + \epsilon n_i^r/m_i +  \omega/q))$. Suppose $m_i = \Theta(n_i^{r - 2\delta})$, then
this simplifies to $O(\epsilon (n_i^{\delta} \log n_i + \epsilon n_i^{2 \delta} +  \omega/q))$. Thus, similar to the previous example, if we take $\epsilon = \Theta(n_i^{-\delta})$, we see that with sample size $q = \Theta(n_i^{8 \delta} \log n_i)$ we can get $|\mathcal{F}/n_i - \mathcal{F}_Q/q_i|$ arbitrarily small.
\end{example}

\subsection{Application to Sublinear Time Algorithms}
Given any algorithm for estimating the 
free energy of an Ising model, the sample complexity results from the previous section suggest a natural way to compute the free energy more efficiently on large graphs: sample a few small subsets of the graph randomly, run the original
algorithm on each of the small sample graphs, and finally return the median of the sample outputs. We analyze the performance of the resulting algorithm
in a few particularly interesting cases. 

As noted in \cref{example:uniform-edge-weights}, if we want to estimate say $\F/n$ to high accuracy and our model is not sufficiently dense, we may sometimes want to take $\epsilon$ shrinking as a function of $n$. However, we will state
the results for general $\epsilon$ and $n$ without assuming anything about their relationship. Similarly, when we say \emph{constant-time}, we mean constant time for fixed $\epsilon$; even when $\epsilon$ is shrinking like $n^{-\delta}$, this may still correspond to a sublinear time algorithm for $\delta$ small (for example, in \cref{thm:sampling-jerrum-sinclair}).

First, we consider the case of ferromagnetic $J$. The result of Jerrum and Sinclair \cite{JerrumSinclair:90} shows we can estimate the free energy (indeed, even the partition function) in $poly(n,1/\epsilon)$ time. On the other
hand, in constant time, it was shown in \cite{previous-paper} that we can estimate the free energy to $\epsilon n \|J\|_F$ error in time $2^{O(\log(1/\epsilon)/\epsilon^2)}$ which is exponential in $\epsilon$. We can give a much better constant time algorithm by combining our sampling approach with the algorithm
of Jerrum and Sinclair; indeed applying \cref{thm:sample-complexity-free-energy} we get the following result as an immediate corollary.

\begin{theorem}\label{thm:sampling-jerrum-sinclair}
Fix $\delta > 0$.
Let $\epsilon > 0$ and suppose $q \ge 128000\omega$, where $\omega:=\log(1/\epsilon)/\epsilon^{8}$. 
Suppose also that $J$ is ferromagnetic, i.e. $J_{ij} \ge 0$ for all $i,j$.
Then, there is an algorithm which runs in time $poly(1/\epsilon)\log(1/\delta)$ and has a vertex sample complexity of $O(q\log(1/\delta))$ which returns an estimate $\hat{F}$ such that
$$\left|\F - \hat{\F}\right| \leq 4001 \epsilon n \left(\|J\|_F + \epsilon n \|\vec{J}\|_{\infty} + \omega/q \right)$$
with probability at least $1 - \delta$.
\end{theorem}

In \cite{previous-paper} we gave a constant time regularity-based algorithm to compute the free energy of a Markov random field. Unfortunately, to compute
an approximation with additive error $\epsilon n \|J\|_F$ it required time $2^{O(1/\epsilon^{2r - 2})}$, whereas we knew that if we allowed for polynomial
time in $n$, the correct exponent for $\epsilon$ does not depend on $r$ at all. Combining the latter result (Theorem 1.17) with our sampling algorithm gives a constant-time algorithm for computing the free energy with similar guarantees but requiring, for fixed $r$, only time $2^{O(1/\epsilon^2)}$.

\begin{theorem}
Let $J$ be an order $r$ Markov Random Field.
Let $\delta,\epsilon > 0$ and suppose $q \ge 10^6\omega$, where $\omega:= r^7\log(1/\epsilon)/\epsilon^{8}$. 
Then,  there is an algorithm which runs in time $2^{O(\log(1/\epsilon)/\epsilon^2)}\log(1/\delta)$ and has a vertex sample complexity of $O(q\log(1/\delta))$ which returns an estimate $\hat{F}$ such that:
$$\left|\F - \hat \F \right| \leq 10^5r^3 \epsilon \left(\sum_{d = 1}^r n^{d/2} \left(\|J_{=d}\|_F + \epsilon n^{d/2} \|\vec{J}\|_{\infty}\right) + \omega n/q\right)$$
with probability at least $1 - \delta$.
\end{theorem}
As previously mentioned, these algorithms for estimating the free energy
immediately imply similar results for estimating the magnetization:
see \cref{appendix-magnetization-proof}.
\subsection{The mean-field approximation and the variational free energy}
The mean-field approximation to the free energy (also referred to as the \emph{variational free energy}) is obtained by restricting the distributions $\mu$ in the variational characterization of the free energy (\cref{eqn:free-energy-variational-char}) to be product distributions. Accordingly, we define the \emph{variational free energy} by 
\[ \F^* := \max_{x \in [-1,1]^n} \left[\sum_{i,j} J_{ij}
      x_i x_j + \sum_i H\left(\frac{x_i +
        1}{2}\right)\right]. \] 

Indeed, if $\bar{x} = (\bar{x}_1,\dots,\bar{x}_n)$ is the optimizer in the above definition, then the product distribution $\nu$ on the boolean hypercube, with the $i^{th}$ coordinate having expected value $\bar{x}_i$, minimizes $\KL(\mu||P)$ among all product distributions $\mu$. Moreover, it is immediately seen from \cref{eqn:free-energy-KL} that the value of this minimum KL is exactly $\F - \F^*$. Thus, the quantity $\F - \F^*$, which measures the quality of the mean-field approximation, may be interpreted information theoretically as the divergence between the closest product distribution to the Boltzmann distribution and the Boltzmann distribution itself.  

We will rely crucially on the following bound on the error
of the mean-field approximation, proved in \cite{previous-paper}:
\begin{theorem}[\cite{previous-paper}]\label{thm-main-structural-result} 
Fix an Ising model $J$ on $n$ vertices.
Let $\nu := \arg\min_{\nu} \KL(\nu || P)$, where $P$ is the Boltzmann distribution and the minimum ranges
over all product distributions. 
Then,
$$ \KL(\nu || P)  = \F - \F^{*} \leq 200 n^{2/3} \|J\|_F^{2/3} \log^{1/3}(n \|J\|_F + e).$$
\end{theorem}
This result provides a key bridge between
the \emph{combinatorial} definition of the free energy
(as a sum over states) and tools in optimization, such
as convex duality, which will be essential to proving our result.
Crucially for our application, this bound is tight enough to show the free energy and variational free energy are close even on relatively small graphs. For a discussion of previous
results in this area, see \cite{previous-paper}. We will deduce \cref{thm:sample-complexity-free-energy} from this bound and the following theorem on the sample complexity of variational free energy. 

\begin{theorem}
\label{thm:sample-complexity-variational-free-energy}
Let $\epsilon > 0$ and suppose $q \ge 128000\omega$, where $\omega:=\log(1/\epsilon)/\epsilon^{8}$. 
Then, with probability at least $39/40$:
$$\left|\F^* - \frac{n}{q}\F^*_Q\right| \leq 2000\epsilon n \left(\|J\|_F + \epsilon n \|\vec{J}\|_{\infty} + \omega/q \right).$$
\end{theorem}

\subsection{Connection to graph limits}
A \emph{graphon} is a symmetric measurable function $W : [0,1]^2 \to [0,1]$ which serves as a natural limiting object for dense graphs; for a proper introduction see the textbook \cite{lovasz2012large}. To a graphon $W$, we can associate 
a natural probability distribution over graphs of size
$n$ defined by the following sampling process:
\begin{enumerate}
\item Sample $u_1, \ldots, u_n \sim \text{Uniform}([0,1])$.
\item Independently include edge $(i,j)$ with probability $W(u_i,u_j)$.
\end{enumerate}
Conversely, there is a natural way to associate a (0-1 valued) graphon $W_G$ to a graph $G$ of size $n$: let $A$ be the $n \times n$ adjacency matrix of $G$,
and let the corresponding graphon $W_G$ be given by splitting $[0,1]^2$ into
$n^2$ equally sized squares on a grid labeled by coordinates $(i,j)$, and setting $W_G$ to be equal to the constant $A_{ij}$ (either 0 or 1) in square $(i,j)$.
In this context, the natural statistical question to study is that of \emph{parameter estimation}: given a graphon parameter $f(W)$ and $\epsilon > 0$, how large of a graph do we need to sample from $W$ in order to estimate $f(W)$ within $\epsilon$-additive error with high probability? In \cite{borgs2008convergent}, necessary and sufficient conditions for a parameter $f$ to be estimable by finite sample size were developed, and it was shown further shown that if $f$ is Lipschitz with respect to the graphon cut metric, then $2^{O(1/\epsilon^2)}$ samples suffice.

As an example, associate to every graph $G$ on $n$ vertices an Ising model by assigning each edge the same weight $\beta/n$, where $\beta > 0$ is fixed. Then, for any graph $G$, we can ask what the free energy of the corresponding Ising model is.   Naively, we cannot apply the graphon theory because
the free energy $\mathcal{F}$ of a graph $G$ cannot be defined solely in terms of its graphon $W_G$. However, it was shown in \cite{borgs2012convergent} that the variational free energy $\mathcal{F}^*$ can still be defined, and that the free energy densities
$\mathcal{F}/n$ and $\mathcal{F}^*/n$ agree in the limit as graph size goes to infinity (see Theorem 5.8 of \cite{borgs2012convergent}); thus the \emph{free energy density of a graphon} can be well-defined\footnote{There are fundamental links between free energies in statistical physics and notions of graph limit convergence which are beyond the scope of this brief summary. The interested reader should consult \cite{borgs2012convergent} for details.}. In the
context of our example, they show that for $\beta$ fixed and for the corresponding Ising models on an arbitrary sequence of graphs $(G_n)$ of increasing size, 
$|\mathcal{F}(G_n)/n - \mathcal{F}^*(G_n)/n| = O(1/\sqrt{\log n})$. In \cite{previous-paper} we improved this rate of convergence considerably to $\tilde{O}(1/n^{1/3})$.

Because the (variational) free energy is also Lipschitz with respect to the graphon cut metric, the result of \cite{borgs2012convergent} shows that the free energy density of a graphon can be estimated to error $\epsilon$ by sampling a graph of size $2^{O(1/\epsilon^2)}$ from $W$ and computing the free energy on this graph. The main result of this paper (\cref{thm:sample-complexity-free-energy}) improves this significantly: it shows that the free
energy density of a graphon can be estimated to error $\epsilon$ by sampling a graph of size only $poly(1/\epsilon)$. Furthermore, given a sampling oracle for the graphon, we also get constant time algorithms for estimating the graphon free energy density: in ferromagnetic or high temperature settings we provide a $poly(1/\epsilon)$ time algorithm, and in the general setting, we provide a $2^{\tilde{O}(1/\epsilon^2)}$ time algorithm. Finally, we remark that our techniques extend in a straightforward manner to deal with higher order Markov random fields, whereas the theory of hypergraph limits is significantly more involved.

\subsection{Overview of the techniques}
As mentioned in the introduction, we will prove our main result (\cref{thm:sample-complexity-free-energy}) by instead proving the corresponding statement for variational free energy (\cref{thm:sample-complexity-variational-free-energy}). That this suffices is guaranteed by \cref{thm-main-structural-result}; crucially this non-asymptotic bound will provide a good bound on the error even on the small sampled graph. As we will see, working the variational free energy instead of the (combinatorial) free energy seems to be essential for our argument to work.

The next step in our argument is to reduce to proving the statement about variational free energy only for interaction matrices which can be written as a sum of a small number of rank one matrices (we refer to such matrices as generalized cut matrices of low rank). This reduction is based on the following two key ingredients. First, the weak regularity lemma of Frieze and Kannan shows that any interaction matrix may be well approximated in a suitable sense by a generalized cut matrix of low rank;  the notion of this approximation is sufficient for the purpose of approximating the free energy (\cref{lemma: free-energy-lipschitz}). Second, a theorem of Alon et al. from \cite{alon-etal-samplingCSP} on the cut norm of random subarrays shows that if two matrices are sufficiently close (in the above sense), then with high probability, random submatrices of a sufficiently large size will also be close. In particular this shows the regularity decomposition of a matrix remains a good approximation in cut norm, even after restriction to the random submatrix corresponding to our sample.

This reduction prepares us for the main technical content of this paper, \cref{section-sample-complexity-generalized-cut}, where we prove the desired sample complexity bound for generalized cut matrices of low rank. For such matrices $D$, the non-entropy part of the variational free energy $x^T D x$ depends only on a small number of statistics of $x$. Moreover, as \cref{lemma:gamma-def} shows, it suffices to know these statistics up to some constant precision. With this, it is quite easy to see (\cref{lemma:easy-direction}) that the rescaled free energy of the sample cannot be much smaller than the free energy of the original graph: this is seen just by restricting
the optimal product distribution on the original graph to the sample. The other direction is harder: we need to rule out the existence of distributions on the sample with unexpectedly large free energy. 

In \cref{prop:approximating-by-finitely-many}, we use the considerations of the previous paragraph to show that up to a small error, the optimization problem defining the variational free energy can be replaced by a small number of maximum-entropy programs  with linear constraints  (\cref{prop:approximating-by-finitely-many}). Note our maximum-entropy programs range only over the space of product distributions; this is significantly different than attempting to optimize over all distributions, the setting in e.g. \cite{singh-vishnoi}. Our strategy will be to show that with high probability, the optimum of each of these programs is not much smaller than the rescaled optimum of the corresponding program for the sample. The fact that there are only a small number of programs will allow us to use the union bound to complete the proof. This part of our proof may be of independent interest. Note that this amounts to showing that the absence of a good solution for the original program implies the absence of good solutions for random induced programs. 

As in \cite{alon-etal-samplingCSP}, our solution will be to use duality: we will use the random restriction of a dual certificate -- which shows that the original program has no good solutions -- to show that with high probability, random induced programs also have no good solutions. However, in the case of \cite{alon-etal-samplingCSP}, a relatively simple application of linear programming duality, to show that infeasible programs continued to stay infeasible, sufficed to show polynomial bounds\footnote{For this simple argument see the conference version \cite{alon-etal-samplingCSP-conference}. In the journal version the LP objective is in fact used to improve the bounds, which makes the argument considerably more complex.}; in our case the objective function is very important, so we have to use convex duality which leads to some rather delicate issues. 

First of all, it is not \emph{a priori} clear that the dual certificate for the original program will actually provide a useful lower bound on the random induced program --- in general the objective of the dual program may depend on its variables in a complex way, and there is no general reason 
that the lower bound we get from reusing the certificate
will actually be of the desired form, or that it will concentrate sufficiently well. Here, we must use the fact that the dual of the maximum entropy program of product distributions with linear constraints has a particularly nice form (\cref{eqn:explicit-form-x(y)}) which behaves well with respect to random restrictions.
Second of all, in order to get concentration of the dual objective, we also need to ensure that none of the coordinates of the dual certificate can influence the objective too much. For this, we use Sion's generalization of Von Neumann's minimax theorem to show that a version of the dual with bounded entries is sufficiently good for our purpose (\cref{lemma:modified-strong-duality}). That this bound on the entries is useful relies on the parameters guaranteed by the weak regularity lemma. Together these considerations allows the analysis to go through (\cref{lemma:upper-bound-on-sample-dual}, \cref{lemma:upper-bound-sample-primal}). 
The proof of the statement for general Markov random fields is similar, and we will omit details.       

\subsection{Acknowledgements}
We thank David Gamarnik for insightful comments, Andrej Risteski for helpful discussions related to his work \cite{risteski-ising}, and Yufei Zhao for introducing us to reference \cite{alon-etal-samplingCSP-conference}.

\section{Preliminaries}

We will make essential use of the weak regularity lemma of Frieze and Kannan \cite{frieze-kannan-matrix}. Before stating
it, we introduce some terminology. Throughout this section, we will
deal with $m\times n$ matrices whose entries we will index by $[m]\times[n]$,
where $[k]=\{1,\dots,k\}$. 
\begin{defn}
Given $S\subseteq[m]$, $T\subseteq[n]$ and $d\in\R$, we define
the $[m]\times[n]$ \emph{Cut Matrix }$C=CUT(S,T,d)$ by 
\[
C(i,j)=\begin{cases}
d & \text{if }(i,j)\in S\times T\\
0 & \text{otherwise}
\end{cases}
\]
\end{defn}

\begin{defn}
A \emph{Cut Decomposition }expresses a matrix $J$ as 
\[
J=D^{(1)}+\dots+D^{(s)}+W
\]

where $D^{(i)}=CUT(R_{i},C_{i},d_{i})$ for all $t=1,\dots,s$. 
We say that such a cut decomposition has \emph{width }$s$\emph{,
coefficient length $(d_{1}^{2}+\dots+d_{s}^{2})^{1/2}$ }and \emph{error
$\|W\|_{\infty\mapsto1}$}.
\end{defn}

We are now ready to state the weak regularity lemma of Frieze and Kannan. The particular choice of constants can be found in \cite{}. 
\begin{theorem}
\label{fk}
\cite{frieze-kannan-matrix}
Let $J$ be an arbitrary real matrix, and let $\epsilon>0$.
Then, we can find a cut decomposition of width at most $16/\epsilon^{2}$, 
coefficient length at most $4\|J\|_{F}/\sqrt{mn}$, error at most $4\epsilon\sqrt{mn}\|J\|_{F}$, and such that $\|W\|_{F}\leq\|J\|_{F}$.  
\end{theorem}

\begin{remark}
\label{rmk:infty-norm-bound}
In particular, we have $$\|\vec{W}\|_{\infty}\leq \|\vec{J}\|_{\infty}+|d_{1}|+\dots+|d_{s}|\leq||\vec{J}||_{\infty}+\sqrt{s}(d_{1}^{2}+\dots+d_{s}^{2})^{1/2}\leq||\vec{J}||_{\infty}+\sqrt{16s}\|J\|_{F}/\sqrt{mn}.$$

\end{remark}

\begin{defn}
We say that $D$ is a \emph{generalized cut matrix of rank s}  if it is possible to express $D$ as a sum of $s$ cut matrices.  
\end{defn}

Our reduction from general matrices to generalized cut matrices of low rank will be based on two ingredients. The first is a simple lemma showing that the variational free energy is $1$-Lipschitz with respect to the cut norm of the matrix of interaction strengths (see, e.g., \cite{}). 
\begin{lemma}
\label{lemma: free-energy-lipschitz} 
Let $J$ and $D$ be the matrices of interaction strengths of 
Ising models with variational free energies $\F^{*}$ and $\F^{*}_{D}$. Then, with $W:= J - D$, we have  
$|\F^{*}-\F^{*}_{D}|\leq\|W\|_{\infty \mapsto 1}$. 
\end{lemma}
\begin{proof}
Note that for any $x\in[-1,1]^{n}$, we have 
\begin{align*}
|\sum_{i,j}J_{i,j}x_{i}x_{j}-\sum_{i,j}D_{i,j}x_{i}x_{j}| & =|\sum_{i}(\sum_{j}W_{i,j}x_{j})x_{i}| \leq|\sum_{i}|\sum_{j}W_{i,j}x_{j}|\\
 & \leq\|W\|_{\infty\mapsto1}, 
\end{align*}
from which we immediately get that $|\F^{*}-\F^{*}_{D}|\leq\|W\|_{\infty \mapsto 1}$. 
\end{proof}

The second ingredient is the following theorem (with $r = 2$) of Alon et al\footnote{Here $\|G\|_{\infty \to 1}$ denotes the supremum of $G(\cdot,\ldots,\cdot)$ on the hypercube $\{\pm 1\}^n$, essentially the cut norm.}. 

\begin{theorem}
\label{thm-alon-et-al-sampling-cutnorm}
\cite{alon-etal-samplingCSP-conference}
Suppose $G$
is an $r$-dimensional array on $V^{r}=V\times V\times\dots\times V$
with all entries of absolute value at most $M$. Let $Q$ be a random
subset of $V$ of cardinality $q\geq1000r^{7}/\varepsilon^{6}$. Let
$B$ be the $r$-dimensional array obtained by restricting $G$ to
$Q^{r}$. Then, with probability at least $39/40$, we get 
\[
\frac{1}{4}\|B\|_{\infty \mapsto 1}\leq\frac{q^{r}}{|V|^{r}} \|G\|_{\infty \mapsto 1} + 10\varepsilon^{2}Mq^{r}+5\varepsilon q^{r}\frac{\|G\|_{F}}{|V|^{r/2}}.
\]
\end{theorem}

\section{Sample complexity for generalized cut matrices}
\label{section-sample-complexity-generalized-cut}
Throughout this section, $D = D^{(1)} + \dots + D^{(s)}$ will denote a generalized $n\times n$ cut matrix where $D^{(i)}=CUT(R_i, C_i, d_i)$ for all $i\in [s]$ and $(d_{1}^{2}+\dots+d_{s}^{2})^{1/2} \leq \alpha/n$ for some $\alpha > 0$. For us, the advantage of working with generalized cut matrices is that for any $x\in [-1,1]^{n}$, the quantity $x^{T}Dx$ depends only on a few statistics of the vector $x$. Indeed, it is readily seen that:
\begin{equation}
\label{eqn:quadratic-simplified-generalized-cut}
\sum_{i,j=1}^{n} D_{i,j}x_ix_j = \sum_{i=1}^{s}r_{i}(x)c_{i}(x)d_{i}, 
\end{equation}
where $r_{i}(x)=\sum_{a\in R_{i}}x_{a}$ and $c_{i}(x)=\sum_{b\in C_{i}}x_{b}$. 

The next lemma shows that for approximating $x^{T}Dx$, it suffices to know the vectors $r(x):=(r_1(x),\dots,r_s(x))$ and $c(x):=(c_1(x),\dots,c_s(x))$ up to some constant precision. 
\begin{lemma}\label{lemma:gamma-def}
Let $D = D^{(1)}+\dots + D^{(s)}$ be a generalized cut matrix as above. Then, given real numbers $r_{i},r'_{i},c_{i},c'_{i}$ for each $i\in[s]$ and some 
$\gamma \in (0,1)$ such that $|r_i|,|c_i|,|r'_i|,|c'_i| \le n$, 
$|r_i - r'_i| \le \gamma n$ and $|c_i - c'_i| \le \gamma n$ 
for all $i\in[s]$, we get that 
$\sum_i d_i|r'_i c'_i - r_i c_i| \le 2\alpha\gamma ns^{1/2}$. 
\end{lemma}
\begin{proof}
Since $ |r'_i c'_i - r_i c_i| \le |c'_i||r'_i - r_i| + |r_i||c'_i - c_i| \le 2\gamma n^2$, it follows by Cauchy-Schwarz that
\begin{align*}
\sum_{i=1}^{s} d_i|r'_i c'_i - r_i c_i|
\le \left(\sum_i d_i ^{2}\right)^{1/2} 2 s^{1/2} \gamma n^{2}
\le  2\alpha\gamma ns^{1/2}.
\end{align*} 
\end{proof}

Since our goal is to approximate the maximum value of $x^{T}Dx + \sum_{i=1}^{n}H((1+x_i)/2)$ as $x$ ranges over $[-1,1]^{n}$, the next definition is quite natural given the previous lemma. For $r:=(r_1,\dots,r_s) \in [-n,n]^s$, $c:=(c_1,\dots,c_s)\in [-n,n]^s$, and $\gamma > 0$, consider the following max-entropy program $\mathcal{C}_{r,c,\gamma}$: 
\begin{align*}
\max & \quad \sum_{i=1}^{n}H\left(\frac{1+x_{i}}{2}\right)\\
s.t.\\
\forall i\in[n]: & \quad -1\leq x_{i}\leq1 \\
\forall t\in[s]: & \quad r_{t}-\gamma n\leq\sum_{i\in R_{t}}x_{i}\leq r_{t}+\gamma n\\
\forall t\in[s]: & \quad c_{t}-\gamma n\leq\sum_{i\in C_{t}}x_{i}\leq c_{t}+\gamma n
\end{align*}
By taking $H(z) = -\infty$ for $z\notin [0,1]$, we may drop the $-1\leq x_i \leq 1$ constraints. We will denote the optimum of this program by $O_{r,c, \gamma}$. We also define 
$$\F^*_{r,c,\gamma}:= \sum_{i=1}^{s}r_ic_id_i + O_{r,c,\gamma}.$$

Let $I_\gamma$ be an arbitrary minimal collection of points in $[-n,n]$ such that every $z \in [-n,n]$ is within distance $\gamma n$ of some element of $I_\gamma$. Clearly, we have $|I_{\gamma}| \le 1/\gamma + 1$. For $\ell \geq 1$, let $\I_{\gamma,\ell} \subseteq I_{\gamma}^{s}\times I_{\gamma}^{s}$ denote the set of pairs $(r,c)\in I_{\gamma}^{s}\times I_{\gamma}^{s}$ for which $O_{r,c,\ell \gamma} \geq 0$.    

The following proposition shows that maximizing $\F^*_{r,c,\ell \gamma}$ over all $(r,c)\in \I_{\gamma,\ell}$ provides a good approximation to $\F^{*}_{D}$. 
\begin{proposition}
\label{prop:approximating-by-finitely-many}
$-2\alpha\ell \gamma ns^{1/2} \leq \F^{*}_{D}-\max_{(r,c)\in \I_{\gamma,\ell}}\F^*_{r,c,\ell\gamma} \leq 2\alpha\ell \gamma ns^{1/2}$
\end{proposition}
\begin{proof}
For the right inequality, let $x^{*}\in [-1,1]^{n}$ denote the vector attaining $\F_{D}^{*}$, and let $r,c\in I_{\gamma}^{s}$
be such that $|r_{i}(x^*)-r_{i}|\leq\ell\gamma n$ and $|c_{i}(x^*)-c_{i}|\leq\ell\gamma n$
for all $i\in[s]$. In particular, we have $O_{r,c,\ell \gamma} \geq \sum_{i=1}^{n}H((1+x^*_i)/2) \geq 0$, so that $(r,c)\in \I_{\gamma,\ell}$.  Then, we have 
\begin{align*}
\F_{D}^{*} & =\sum_{i=1}^{s}r_{i}(x^*)c_{i}(x^*)d_{i}+\sum_{i=1}^{n}H\left(\frac{1+x^*_{i}}{2}\right)\\
 & \leq\sum_{i=1}^{s}r_{i}(x^*)c_{i}(x^*)d_{i}+O_{r,c,\ell \gamma}\\
 & \leq\sum_{i=1}^{s}r_{i}c_{i}d_{i}+2\alpha\ell \gamma ns^{1/2}+O_{r,c,\ell \gamma}\\
 & = \F^*_{r,c,\gamma}+2\alpha \gamma ns^{1/2}\\
 & \leq \max_{(r,c)\in \I_{\gamma,\ell}}\F^*_{r,c,\gamma}+2\alpha \gamma ns^{1/2},
\end{align*}
where in the first line we have used \cref{eqn:quadratic-simplified-generalized-cut}, and in the third line we have used \cref{lemma:gamma-def}.

For the left inequality, we will show that $\F^*_{r,c,\ell\gamma}\leq\F_{D}^{*}+2\alpha \ell\gamma ns^{1/2}$
for all $(r,c)\in I_{\gamma}^{s}\times I_{\gamma}^{s}$. Accordingly, fix $(r,c)\in I_{\gamma}^{s}\times I_{\gamma}^{s}$,
and let $x_{r,c}\in[-1,1]^{n}$ denote a point attaining $O_{r,c,\ell\gamma}$
(if no such point exists, then $O_{r,c,\ell\gamma}=-\infty$ and we are
trivially done). Then, by the same computation as above, we get 
\begin{align*}
\F^*_{r,c,\gamma} & =\sum_{i=1}^{s}r_{i}c_{i}d_{i}+\sum_{i=1}^{n}H\left(\frac{1+x_{r,c}}{2}\right)\\
 & \leq\sum_{i=1}^{s}r_{i}(x_{r,c})c_{i}(x_{r,c})d_{i}+2\alpha\ell \gamma ns^{1/2}+\sum_{i=1}^{n}H\left(\frac{1+x_{r,c}}{2}\right)\\
 & \leq\sum_{i=1}^{s}r_{i}(x^{*})c_{i}(x^{*})d_{i}+\sum_{i=1}^{n}H\left(\frac{1+x^{*}}{2}\right)+2\alpha\ell \gamma ns^{1/2}\\
 & \leq\F_{D}^{*}+2\alpha\ell \gamma ns^{1/2}.
\end{align*}
\end{proof}

The remainder of this section will be devoted to proving \cref{prop:sample-complexity-generalised-cut-matrices}, which is a version of \cref{thm:sample-complexity-variational-free-energy} for generalized cut matrices, and will be used crucially in the proofs of our main results. Before stating it, we need to introduce some more notation. 

Let $Q$ denote a random subset of $[n]$ of size $|Q|=q$. Let $\tilde{D}:=\frac{n}{q}D$ and let $\tilde{D}_{Q}$ denote the matix induced by $\tilde{D}$ on $Q\times Q$. In particular, note that we can write 
$$\tilde{D}_{Q} = \tilde{D}^{(1)}_{Q}+\dots + \tilde{D}^{(s)}_{Q},$$
where $\tilde{D}^{(i)}_{Q} = CUT(R_i \cap Q, C_i \cap Q, \tilde{d_i})$ for all $i \in [s]$, with $\tilde{d_i}:=\frac{n}{q}d_i$. We will also make use of the corresponding max-entropy program $C(Q)_{r,c,\gamma}$ (for $r,c \in [-n,n]^s$):  
\begin{align*}
\max & \quad \sum_{i\in Q}H\left(\frac{1+x_{i}}{2}\right)\\
s.t.\\
\forall i\in Q & \quad  -1\leq x_{i}\leq1\\
\forall t\in[s]: & \quad  r'_t-\gamma q\leq\sum_{j\in R_{t}\cap Q}x_{j}\leq r'_t+\gamma q\\
\forall t\in[s]: & \quad c'_t-\gamma q\leq\sum_{j\in C_{t}\cap Q}x_{j}\leq c'_t+\gamma q,
\end{align*}
where $r'= \frac{q}{n}r$ and $c'=\frac{q}{n}c$. We will denote the optimum of this program by $O(Q)_{r,c,\gamma}$. As before, let 
$$ \F^*(Q)_{r,c,\gamma} := \sum_{i=1}^{s}r'_ic'_i\tilde{d_i} + O(Q)_{r,c,\gamma},$$
let $\I(Q)_{\gamma,\ell} \subseteq I_{\gamma}^{s}\times I_{\gamma}^{s}$ denote the set of pairs $(r,c)\in I_{\gamma}^{s}\times I_{\gamma}^{s}$ for which $O(Q)_{r,c,\ell \gamma} \geq 0$, and note that \cref{prop:approximating-by-finitely-many} shows that      
\begin{equation}
\label{eqn:approx-by-finite-sample}
\left|\F^{*}_{\tilde{D}_{Q}} - \max_{(r,c)\in \I(Q)_{\gamma,\ell}}\F^*(Q)_{r,c,\gamma} \right| \leq 2\alpha \ell \gamma qs^{1/2}. 
\end{equation}
The goal of the next few sections will be to relate the free energy of the full
graph and its sampled version as follows:
\begin{proposition}
\label{prop:sample-complexity-generalised-cut-matrices}
Suppose $2\alpha \gamma s^{1/2} < 1$. Then, $\left|\F^{*}_{D} - \frac{n}{q}\F^{*}_{\tilde{D}_Q}\right|\leq 8\alpha \gamma n s^{1/2}$, except with probability at most $\exp(-2\alpha^2 \gamma ^2 sq) + 4s\exp(-2\gamma^2 q) + 2\exp(-\alpha^2 \gamma ^4 q/32s)\exp(2s\log(2/\gamma))$ over the choice of $Q$. 
\end{proposition}
We begin by proving the easier direction of the above inequality:
\begin{lemma}\label{lemma:easy-direction}
$\frac{n}{q}\F^{*}_{\tilde{D}_Q} \geq \F^{*}_{D} - 3\alpha \gamma n s^{1/2}$, except with probability at most $\exp(-2\alpha^2 \gamma^2 sq) + 4s\exp(-2\gamma^2 q)$.
\end{lemma}
\begin{proof}
Let $x^*\in [-1,1]^n$ attain $\F^*_D$, and let $r(x^*)=(r_1(x^*),\dots,r_s(x^*))$, $c(x^*) = (c_1(x^*),\dots, c_s(x^*))$ be as above. Let $x^*_Q$ denote $x^*$ restricted to the vertices in $Q$, and let $r_i(x^*_Q):=\sum_{j\in R_i \cap Q}x^*_Q$, $c_i(x^*_Q):=\sum_{j \in C_i \cap Q}x^*_Q$ for all $i\in [s]$. Then, for any $i \in [s]$, Hoeffding's inequality shows that $\Pr\left[\left|r_i(x^*_Q)-\frac{q}{n}r_i(x^*)\right| \geq \gamma q\right] \leq 2\exp(-2\gamma^{2}q)$, and similarly for $c_i$. Also by Hoeffding's inequality, $\Pr\left[\sum_{j\in Q}H(x^*_j)-\frac{q}{n}\sum_{i=1}^{n}H(x^*_i) \leq -\alpha\gamma qs^{1/2}\right] \leq \exp(-2\alpha^{2}\gamma^{2}sq)$. Finally, the union bound and \cref{lemma:gamma-def} give the desired conclusion.  
\end{proof}

The upper bound on $\F^{*}_{\tilde{D}_{Q}}$ is more involved, and requires some notions from convex duality which we will review in the next section.
\subsection{Convex duality and application to the maximum entropy problem}
We consider the following general form of the \emph{maximum-entropy problem for product distributions} with linear constraints, henceforth referred to as the \emph{primal}:
\begin{alignat*}{5}
&\sup \quad &\sum_{i = 1}^n &H\left(\frac{1+x_i}{2}\right)\\
&\ s.t. \quad &a_j \cdot x - b_j &\le 0 && \forall j \in [m], \\
\end{alignat*}
where $H(z)$ is the binary entropy function with $H(z):= -\infty$ for $z\notin [0,1]$. We will denote the optimum of this program by $OPT$. 

\begin{remark}
Note that the value of the objective is $-\infty$ if $x\notin [-1,1]^{n}$. Since $\sum_{i=1}^{n} H((1+x_i)/2)$ is strictly concave on the compact, convex set $[-1,1]^{n}$, it follows that either $OPT = -\infty$ or $OPT > -\infty$ is attained by a unique point in $[-1,1]^{n}$.    
\end{remark}
We define the \emph{Lagrangian} by
\[ L(x,y) := \sum_{i = 1}^n H\left(\frac{1+x_i}{2}\right) - \sum_{j = 1}^m y_j (a_j \cdot x - b_j), \]
and the \emph{Lagrange dual function} by
\[ g(y) := \sup_{x\in \R^{n}} L(x,y) = \max_{x\in [-1,1]^{n}}\left\{\sum_{i = 1}^n H\left(\frac{1+x_i}{2}\right) - \sum_{j = 1}^m y_j (a_j \cdot x - b_j)\right\}. \]
Note that $g(y)$ is a supremum of linear functions in $y$, hence convex.
We will denote $\arg\max_{x\in [-1,1]^n}L(x,y)$ by $x(y)$, so
\[ g(y) = \sum_{i = 1}^n H\left(\frac{1+x_i(y)}{2}\right) - \sum_{j = 1}^m y_j (a_j \cdot x(y) - b_j). \]
We have the following explicit formula:
\begin{equation}
\label{eqn:explicit-form-x(y)}
x_i(y) = \tanh\left(-\sum_{j=1}^{m}y_{j}a_{j,i}\right) = 2\sigma\left(-2\sum_{j=1}^{m}y_{j}a_{j,i}\right)-1,
\end{equation}
where $\sigma(z):=1/(1+e^{-z})$ is the usual sigmoid function, since the point defined by the right hand side is readily seen to be the maximizer of the strictly concave function $x\mapsto L(x,y)$ on the convex set $[-1,1]^{n}$. In particular, note that $x_i(y)$ depends only on those $a_{j,k}$ for which $k=i$. 

Observe that for any $y \ge 0$, $g(y) \ge OPT$. Indeed, for $x^*$ attaining the primal optimum, we have  
\begin{equation}\label{weak-duality}
g(y) \ge \sum_{i = 1}^n H\left(\frac{1+x_i^*}{2}\right) - \sum_{j = 1}^m y_j (a_j \cdot x^* - b_j) \ge \sum_{i = 1}^n H\left(\frac{1+x_i^*}{2}\right) = OPT.
\end{equation}

Based on this, it is natural to define the \emph{Lagrange dual problem}: 
\[ \inf_y g(y) \quad s.t. \quad y \ge 0. \]
We denote the optimum of the dual program by $OPT^*$, and observe that \cref{weak-duality} shows that $OPT^* \geq OPT$. Strong duality for convex programs shows the following proposition holds.
\begin{proposition}
Strong duality holds, i.e. $OPT^* = OPT$.
\end{proposition}
\begin{proof}
Since all the constraints in the primal are affine, Slater's condition for strong convex duality (as in \cite{rockafellar}) immediately shows that $OPT^* = OPT$. 
We provide an alternate proof, which also illustrates some ideas that will be useful later. Observe that $L(x,y)\colon [-1,1]^{n} \times [0,\infty)^{m} \to \mathbb{R}$ is continuous and concave on $[-1,1]^{n}$ for each $y\in [0,\infty)^{m}$, and is continuous and convex on $[0,\infty)^{m}$ for each $x\in [-1,1]^{n}$. Therefore, we have  
\begin{align*}
OPT^* = \inf_{y\geq 0} \max_{x\in [-1,1]^{n}} L(x,y) &= \max_{x \in [-1,1]^{n}} \inf_{y\geq 0} L(x,y) 
\\
&= \max_{x \text{ feasible for primal}} \inf_{y\geq 0}L(x,y) = \max_{x \text{ feasible for primal}} L(x,0) = OPT,
\end{align*}
where in the second equality we have used Sion's generalization of Von Neumann's minimax theorem \cite{sion1958general}, in the third equality we have used that if $x$ is infeasible for the primal, then $\inf_{y\geq 0} L(x,y) = -\infty$ (by blowing up the weight of a violated constraint), and in the last equality, we have used that $\inf_{y\geq 0}L(x,y)=L(x,0)$ for any feasible $x$.
\end{proof}

\subsection{Upper bound on $\F^{*}_{\tilde{D}_Q}$ via convex duality}

Returning to our max-entropy program $\mathcal{C}_{r,c,\gamma}$, observe that the dual program $\mathcal{C}^{*}_{r,c,\gamma}$ is given by
\begin{align*}
\inf & \quad \sum_{i=1}^{n}H\left(\frac{1+x_i(y)}{2}\right)-\sum_{j=1}^{m}y_{j}\left(\sum_{k=1}^{n}a_{j,k}x_k(y)-b_{j}\right)\\
s.t.&\quad y\geq0,
\end{align*}
where $m=4s$; 
for all $j\in [s]$, $a_{j,i}=1_{i\in R_{j}}$,
$a_{s+j,i}=-1_{i\in R_{j}}$, $a_{2s+j,i}=1_{i\in C_{j}}$, $a_{3s+j,i}=-1_{i\in C_{j}}$; for all $j\in [s]$, $b_{j}=r_{j}+\gamma n$, $b_{s+j}=-r_j + \gamma n$, $b_{2s+j}= c_j +\gamma n$, $b_{3s+j} = -c_j+\gamma n$. We will find it more convenient to work with a modified version of the dual program in which $y$ is also bounded from above. Accordingly, we define the program $\mathcal{C}^{*}_{r,c,\gamma,K}$ (with $m$, $a_{j,i}$ and $b_j$ as above): 
\begin{align*}
\inf & \quad \sum_{i=1}^{n}H\left(\frac{1+x_i(y)}{2}\right)-\sum_{j=1}^{m}y_{j}\left(\sum_{k=1}^{n}a_{j,k}x_k(y)-b_{j}\right)\\
s.t.\\
\forall j \in [m]: &\quad 0\leq y_j \leq K/\gamma.
\end{align*}
The next lemma is the replacement for strong duality that we will use in this setup. 
\begin{lemma}
\label{lemma:modified-strong-duality}
Let $O^{*}_{r,c,\gamma,K}$ denote the optimum of the program $\mathcal{C}^{*}_{r,c,\gamma,K}$. Then, $$O_{r,c,\gamma} \leq O^{*}_{r,c,\gamma, K} \leq \max\left\{O_{r,c,2\gamma},-(K-1)n\right\}.$$
\end{lemma}
\begin{proof}
The first inequality is immediate from \cref{weak-duality}. For the second inequality, we begin by noting that  
\begin{equation}
\label{eqn:infeasible}
\max_{x\text{ infeasible for }C_{r,c,2\gamma}}\min_{y\in[0,K/\gamma]^{m}}L(x,y) \leq -(K-1)n.
\end{equation}
Indeed, if $x$ is infeasible for $\mathcal{C}_{r,c,2\gamma}$, then $(a_{j_0}.x - b_{j_0}) \geq \gamma n$ for some $j_0\in[m]$, and taking $y = (y_1,\dots,y_m)$ with $y_i = \textbf{1}_{i=j_0}K/\gamma$ gives the desired inequality, since for any $p$ we have
$H(p) \le H(1/2) = \log 2 < 1$.
Thus, we have
\begin{align*}
O_{r,c,\gamma,K}^{*} & =\min_{y\in[0,K/\gamma]^{m}}\max_{x\in[-1,1]^{n}}L(x,y)\\
 & =\max_{x\in[-1,1]^{n}}\min_{y\in[0,K/\gamma]^{m}}L(x,y)\\
 & \leq\max\left\{ \max_{x\text{ feasible for }C_{r,c,2\gamma}}L(x,0),\max_{x\text{ infeasible for }C_{r,c,2\gamma}}\min_{y\in[0,K/\gamma]^{m}}L(x,y)\right\} \\
 & \leq\max\left\{ O_{r,c,2\gamma},-(K-1)n\right\},
\end{align*}
where we have used the generalized minimax theorem in the second line and \cref{eqn:infeasible} in the last line. 


\end{proof}
Similarly, we can define the corresponding program $\mathcal{C(Q)}^{*}_{r,c,\gamma, K}$ with optimum $O(Q)^{*}_{r,c,\gamma, K}$, and note that by \cref{lemma:modified-strong-duality}, 
\begin{equation}
\label{eqn:strong-duality-sample}
O(Q)_{r,c,\gamma} \leq O(Q)^{*}_{r,c,\gamma, K} \leq \max\left\{O(Q)_{r,c,2\gamma},-(K-1)q\right\}.
\end{equation}

The next lemma records the relation between $O(Q)^*_{r,c,\gamma,K}$ and $O^*_{r,c,\gamma,K}$ that we will need.    
\begin{lemma}
\label{lemma:upper-bound-on-sample-dual}
$\frac{n}{q}O(Q)^{*}_{r,c,\gamma, K} \leq O^{*}_{r,c,\gamma, K} + 2n\alpha \gamma s^{1/2}$ with probability at least $1-2\exp\left(-\frac{\alpha^{2}\gamma^{4}q}{8K^2 s}\right)$. 
\end{lemma}
\begin{proof}
Let $y^{*}$ denote the optimizer of $\mathcal{C}^{*}_{r,c,\gamma,K}$, so that 
\[
O^{*}_{r,c,\gamma, K} =
\sum_{i=1}^{n}H\left(\sigma\left(-2\sum_{j=1}^{m}y^*_{j}a_{j,i}\right)\right)-\sum_{j=1}^{m}y^*_{j}\left(\sum_{k=1}^{n}a_{j,k}\tanh\left(-\sum_{j=1}^{m}y^*_{j}a_{j,k}\right)-b_{j}\right).
\]

Moreover, by definition, we have 
$$ O(Q)^{*}_{r,c,\gamma,K}\leq \sum_{i\in Q}H\left(\sigma\left(-2\sum_{j=1}^{m}y^*_{j}a_{j,i}\right)\right)-\sum_{j=1}^{m}y^*_{j}\left(\sum_{k\in Q}a_{j,k}\tanh\left(-\sum_{j=1}^{m}y^*_{j}a_{j,k}\right)-\frac{q}{n}b_{j}\right).$$ 
Finally, we rewrite 
\[
\sum_{j=1}^{m}y_{j}^{*}\sum_{k}a_{j,k}\tanh\left(-\sum_{j=1}^{m}y^*_{j}a_{j,k}\right)=\sum_{k}\sum_{j=1}^{m}y_{j}^{*}a_{j,k}\tanh\left(-\sum_{j=1}^{m}y^*_{j}a_{j,k}\right),
\]
and observe that by Hoeffding's inequality, the following holds:
\begin{align*}
\sum_{i\in Q}H\left(\sigma\left(-2\sum_{j=1}^{m}y_{j}^{*}a_{j,i}\right)\right) & \leq\frac{q}{n}\sum_{i=1}^{n}H\left(\sigma\left(-2\sum_{j=1}^{m}y_{j}^{*}a_{j,i}\right)\right)+ q\alpha\gamma s^{1/2}\\
\sum_{k\in Q}\sum_{j=1}^{m}y_{j}^{*}a_{j,k}\tanh\left(-\sum_{j=1}^{m}y^*_{j}a_{j,k}\right) & \geq\frac{q}{n}\sum_{k=1}^{n}\sum_{j=1}^{m}y_{j}^{*}a_{j,k}\tanh\left(-\sum_{j=1}^{m}y^*_{j}a_{j,k}\right)- q\alpha \gamma s^{1/2},
\end{align*}
except with probability at most $2\exp\left(-\frac{\alpha^{2}\gamma^{4}q}{8K^2 s}\right)$. 

\end{proof}

We need one final lemma before we can prove \cref{prop:sample-complexity-generalised-cut-matrices}. 
\begin{lemma} 
\label{lemma:upper-bound-sample-primal}
Let $2\alpha\gamma s^{1/2} < K-1$. Then, except with probability at most $2\exp(-\alpha^2\gamma^4q/8K^2s)\exp(2s\log(2/\gamma))$ over the choice of $Q$, the following holds: 
\begin{enumerate}
\item $\I(Q)_{\gamma,1} \subseteq \I_{\gamma,2}$,
\item for all $(r,c)\in \I(Q)_{\gamma,1}$, $\frac{n}{q}O(Q)_{r,c,\gamma} \leq O_{r,c,2\gamma} + 2n\alpha \gamma s^{1/2}$, and 
\item $\frac{n}{q}\max_{(r,c)\in \I(Q)_{\gamma,1}} \F^*(Q)_{r,c,\gamma} \leq \max_{(r,c)\in \I_{\gamma,2}}\F^*_{r,c,2\gamma}+ 2n\alpha \gamma s^{1/2}$.
\end{enumerate}
\end{lemma}
\begin{proof}
By \cref{lemma:modified-strong-duality}, \cref{eqn:strong-duality-sample} and \cref{lemma:upper-bound-on-sample-dual}, it follows that for any particular $(r,c) \in I_{\gamma}^{s}\times I_{\gamma}^{s}$, 
\begin{equation}
\label{eqn:first-sample-bound}
\frac{n}{q}O(Q)_{r,c,\gamma} \leq \max\left\{O_{r,c,2\gamma},-(K-1)n\right\} + 2n\alpha \gamma s^{1/2}
\end{equation}
except with probability at most $2\exp(-\alpha^2\gamma^4q/8K^2s)$. Since $|I_\gamma | \leq \gamma^{-1} + 1$, it follows by the union bound that \cref{eqn:first-sample-bound} holds simultaneously for all $(r,c) \in I_{\gamma}^{s}\times I_{\gamma}^{s}$ except with probability at most 
$2\exp(-\alpha^2\gamma^4q/8K^2s)\exp(2s\log(2/\gamma))$. We claim that whenever this happens, $1.$, $2.$ and $3.$ hold. 

For $1.$, note that if $(r,c)\notin \I_{\gamma,2}$, then $O_{r,c,2\gamma} = -\infty$. Therefore, \cref{eqn:first-sample-bound}, along with the assumption that $2\alpha \gamma s^{1/2} < K-1$ implies that 
$$ \frac{n}{q}O(Q)_{r,c,\gamma} \leq -(K-1)n + 2n\alpha\gamma s^{1/2} < 0,$$
which shows that $(r,c)\notin \I(Q)_{\gamma,1}$. In particular, if $(r,c) \in \I(Q)_{\gamma,1}$, then $O_{r,c,2\gamma} \geq 0$ so that $\max\{O_{r,c,2\gamma}, -(K-1)n\} = O_{r,c,2\gamma}$. With this, $2.$ follows immediately from \cref{eqn:first-sample-bound}. Finally, $3.$ follows from $2.$, along with the observation that $\frac{n}{q}\sum_{i=1}^{s}r'_ic'_i\tilde{d_i} = \sum_{i=1}^{s}r_ic_id_i$. 
\end{proof}

\begin{proof}[Proof of \cref{prop:sample-complexity-generalised-cut-matrices}]
By conclusion $3.$ of \cref{lemma:upper-bound-sample-primal} (with $K=2$), along with \cref{prop:approximating-by-finitely-many} and \cref{eqn:approx-by-finite-sample}, it follows that except with probability at most $2\exp(-\alpha^2\gamma^4q/32s)\exp(2s\log(2/\gamma))$, we have:
\begin{align*}
\frac{n}{q}\F_{\tilde{D_{Q}}}^{*} & \leq\frac{n}{q}\max_{(r,c)\in\I(Q)_{\gamma,1}}\F^*(Q)_{r,c,\gamma}+2n\alpha\gamma s^{1/2}\\
 & \leq\max_{(r,c)\in\I_{\gamma,2}}\F^*_{r,c,2\gamma}+4n\alpha\gamma s^{1/2}\\
 & \leq\F_{D}^{*}+8n\alpha\gamma s^{1/2}.
\end{align*}

By \cref{lemma:easy-direction}, except with probability at most $ \exp(-2\alpha^{2}\gamma^{2}sq) + 4s\exp(-2\gamma^2q)$, we have that $\frac{n}{q}\F_{\tilde{D_Q}}^{*} \geq \F^{*}_{D} - 3\alpha \gamma n s^{1/2}$. The union bound completes the proof.  
\end{proof}

\section{Proof of \cref{thm:sample-complexity-variational-free-energy}}
Throughout this section, $J$ will denote the matrix of interaction strengths of an Ising model on the vertex set $[n]$, $Q$ will denote a random subset of $[n]$ of size $q$, and $\tilde{J}_Q$ will denote the restriction of $\tilde{J}:=\frac{n}{q}J$ to $Q \times Q$. We will denote the variational free energy corresponding to $J$ by $\F^*$, and the variational free energy corresponding to $\tilde{J}_Q$ by $\F^*_Q$. 
Moreover, we fix $\epsilon >0$ and a cut decomposition $J = D^{(1)} + \dots + D^{(s)} + W$ with parameter $\epsilon$, as guaranteed by \cref{fk}. We will let $D$ denote $D^{(1)} + \dots + D^{(s)}$ and let $\tilde{D}_Q$ denote the restriction of the matrix $\tilde{D}:=\frac{n}{q}D$ to $Q \times Q$. 
\begin{lemma}
\label{lemma:alon-et-al-application}
If $q \geq 128000/\varepsilon^{6}$, then with probability at least $39/40$, we have 
$$\left|\F^*_Q - \F^*_{\tilde{D}_Q}\right| \leq q\|J\|_{F} \left(16\epsilon + 640\varepsilon^{2}\epsilon^{-1} + 20\varepsilon \right) + 40\varepsilon^{2}nq \|J\|_{\infty} $$
\end{lemma}
\begin{proof}
We use \cref{thm-alon-et-al-sampling-cutnorm} with $r=2$ and $G = \tilde{J}-\tilde{D}$. By \cref{fk} and \cref{rmk:infty-norm-bound}, we can take $\|G\|_{\infty\mapsto 1} \leq 4 \epsilon \frac{n^2}{q} \|J\|_{F} $, $M \leq \frac{n}{q}\|J\|_{\infty} + \frac{16}{\epsilon q}\|J\|_{F}$, and $\|G\|_{F} \leq \frac{n}{q} \|J\|_{F}$. Therefore, letting $B:=\tilde{J}_Q - \tilde{D}_Q$, we get that with probability at least $39/40$, 
$$\|B\|_{\infty \mapsto 1} \leq 16\epsilon q \|J\|_{F} +640\varepsilon^{2}q \epsilon^{-1} \|J\|_{F} + 20\varepsilon q \|J\|_{F} + 40\varepsilon^{2}nq \|J\|_{\infty}.$$

Now, a direct application of \cref{lemma: free-energy-lipschitz} completes the proof. 
\end{proof}

\begin{proof}[Proof of \cref{thm:sample-complexity-variational-free-energy}]
By applying \cref{prop:sample-complexity-generalised-cut-matrices} with $q = C\log(1/\epsilon)/\epsilon^{8}$, $\alpha = 4\max\{\|J\|_{F},100/C\}$, $s = 16/\epsilon^2$ and $\gamma = \epsilon$, where $C$ is some constant which is at least $128000$, we see that except with probability at most $1/40$, $$\left| \F^*_D - \frac{n}{q}\F^*_{\tilde{D}_Q}\right| \leq 128 \epsilon \max\{\|J\|_{F},100/C\}n.$$
Further, by applying \cref{lemma:alon-et-al-application} with $q$ as above and $\varepsilon = \epsilon$, we get that except with probability at most $1/40$, 
$$\left|\frac{n}{q}\F^*_{\tilde{D}_Q} - \frac{n}{q}\F^*_{Q}\right| \leq 676 \epsilon \|J\|_{F} n + 40 \epsilon^{2} n^2 \|J\|_{\infty}.$$
Finally, since $\left|\F^* - \F^*_{D} \right| \leq 4\epsilon \|J\|_{F} n $, the triangle inequality and union bound complete the proof. 
\end{proof}
\section{Proof of \cref{thm:sample-complexity-free-energy}}
We continue to use the notation from the previous section.
\begin{proof}
From \cref{thm:sample-complexity-variational-free-energy},
we have
$$\left|\F^* - \frac{n}{q}\F^*_Q\right| \leq 2000\epsilon n \left(\|J\|_F + \epsilon n \|\vec{J}\|_{\infty} + \omega/q \right).$$
Thus, it only remains to bound $|\F - \F^*|$ and $|\F_Q - \F^*_Q|$. Recall
from the definition of variational free energy that $\F - \F^*$ is always
nonnegative so we just need one-sided bounds.
We use the following Lemma from \cite{previous-paper}, which is equivalent
to \cref{thm-main-structural-result}, but more convenient in our situation:
\begin{lemma}[Lemma 3.4 of \cite{previous-paper}]\label{lemma:epsilon-bound}
For any $\epsilon > 0$,
\[ \F - \mathcal{F^*} \le \epsilon n \|J\|_F + 10^5 \log(e + 1/\epsilon)/\epsilon^2. \]
\end{lemma}
To apply this to bound to $\F_Q - \F^*_Q$, we observe that
\[ \E[\|\tilde J_Q\|_F^2] = \|J\|_F^2 \]
so by Markov's inequality,
\[ \|\tilde J_Q\|_F \le 8 \|J\|_F \]
with probability at least $39/40$.
Recall that $\omega = \log(1/\epsilon)/\epsilon^8$. 
Applying Lemma~\ref{lemma:epsilon-bound} with $\epsilon_1=10 \epsilon^2$ to bound both $\F_Q - \F^*_Q$ and $\F - \F^*$,
and using the triangle inequality, we then see that
\[ |\F - \frac{n}{q} \F_Q| \le 4000\epsilon n \left(\|J\|_F + \epsilon n \|\vec{J}\|_{\infty} + \omega/q \right) \]
\end{proof}
\section{Proof of \cref{thm-mrf-sample-complexity}}
\begin{proof}
The proof is essentially same as that of \cref{thm:sample-complexity-free-energy} except that we use a generalized version of the
weak regularity lemma for tensors, as well as a more general bound on the error of the mean-field approximation:
\begin{theorem}
\cite{alon-etal-samplingCSP}\label{reg-alon-etal-mrf}
Let $J$ be an arbitrary $k$-dimensional matrix on $X_{1}\times\dots\times X_{k}$,
where we assume that $k\geq 1$ is fixed. Let $N:=|X_{1}|\times\dots\times|X_{k}|$
and let $\epsilon>0$. Then, in time $2^{O(1/\epsilon^{2})}O(N)$
and with probability at least $0.99$, we can find a cut decomposition of
width at most $4/\epsilon^{2}$, error at most $\epsilon\sqrt{N}\|J\|_F$,
and the following modified bound on coefficient length: $\sum_i |d_i| \le 2\|J\|_F/\epsilon\sqrt{N}$, where $(d_i)_{i =1}^s$ are the coefficients of
the cut arrays.
\end{theorem}
\begin{theorem}\label{thm-mrf-main-structural-result} 
Fix an order $r$ Markov random field $J$ on $n$ vertices.
Let $\nu := \arg\min_{\nu} \KL(\nu || P)$, where $P$ is the Boltzmann distribution and the minimum ranges
over all product distributions.
Then, 
$$ \KL(\nu || P)  = \F - \F^{*} \leq 2000r \max_{1 \le d \le r} d^{1/3}n^{d/3} \|J_{=d}\|_F^{2/3} \log^{1/3}(d^{1/3}n^{d/3} \|J_{=d}\|_F^{2/3} + e).$$
\end{theorem}

The reduction to generalized cut arrays still works: we use the generalized
regularity lemma to decompose each of $J_{=1}, \ldots, J_{=r}$ and then use \cref{thm-alon-et-al-sampling-cutnorm}, taking the union bound for $d$ from $1$ to $r$; in order to boost the success probability of each application to $1 - O(1/r)$, it is more than sufficient to lose a multiplicative factor of $r$ in the bound (refer to the proof in \cite{alon-etal-samplingCSP-conference}). 
From there, as before, we reduce the problem to the maxima
of convex programs by fixing the values of $r(x),c(x)$ up to constant precision,
and then the crucial analysis of convex duality works as before because
we still get a max-entropy problem for a product distribution with linear constraints.
\end{proof}
\newpage 

\bibliographystyle{plain}
\bibliography{ising-regularity,all}

\newpage 

\appendix
\section{Appendix: Estimating the Magnetization from Free Energies }\label{appendix-magnetization-proof}
\begin{theorem}\label{thm-approx-magnetization}
Consider an Ising model 
\[
\Pr[X=x]:=\frac{1}{Z}\exp\{\sum_{i,j}J_{i,j}x_{i}x_{j} + \sum_i h_i x_i\}
\]
Consider also the perturbed models where
\[
\Pr_{h}[X=x]:=\frac{1}{Z}\exp\{\sum_{i,j}J_{i,j}x_{i}x_{j} + \sum_i (h_i+h) x_i\}
\]
and let $m_h$ denote the expected total magnetization for $\Pr_{h}$. Then, for any $\epsilon,\nu > 0$, supposing we have an oracle to compute free energies within error $\epsilon \nu$ for all perturbed models with $|h| \le \nu$, we can find an $\epsilon$ additive approximation to $m_h$, for some $h$ with $|h| < \nu$
while making only 3 queries to the oracle. 
\end{theorem}
Consider the dense case, where we can estimate the free enegy density using a constant size sample.
There is an easy lower bound showing that one cannot, with a constant number of queries, approximate the magnetization for the exact model for each model, so that the extra $h$ is indeed needed in the above statement. 
This is related to the fact that ``symmetry breaking'' is a global phenomenon.  

\begin{proof}
It is well known that one can express the moments of spin systems in terms of derivatives of the log partition function. In particular, for the Ising model $\Pr[X=x]=\frac{1}{Z}\exp\{\sum_{i,j}J_{i,j}x_{i}x_{j}+\sum_{i}h_{i}x_{i}\}$,
consider the family of perturbed Ising models defined by $\Pr_{h}[X=x]=\frac{1}{Z_{h}}\exp\{\sum_{i,j}J_{i,j}x_{i}x_{j}+\sum_{i}(h_{i}+h)x_{i}\}$.
Then, for any $h_{0}$, we have 

\begin{align*}
\frac{\partial\log Z_{h}}{\partial h}(h_{0}) & =\frac{1}{Z_{h_{0}}}\frac{\partial}{\partial h}\left(\sum_{x\in\{\pm1\}^{n}}\exp\{\sum_{i,j}J_{i,j}x_{i}x_{j}+\sum_{i}(h_{i}+h)x_{i}\}\right)\\
 & =\sum_{x\in\{\pm1\}^{n}}\frac{1}{Z_{h_{0}}}\left(\exp\{\sum_{i,j}J_{i,j}x_{i}x_{j}+\sum_{i}(h_{i}+h_{0})x_{i}\}\right)\left(\sum_{i}x_{i}\right)\\
 & =\boldsymbol{E}_{h_{0}}[\sum_{i}x_{i}]
\end{align*}
where $\boldsymbol{E}_{h_{0}}$ denotes the expectation with respect
to the Ising distribution perturbed by $h_{0}$. In particular, $\frac{\partial\log Z_{h}}{\partial h}(0)$
equals the expected total magnetization of the Ising model we started
out with. Moreover, since by Jensen's inequality,
\begin{align*}
\frac{\partial^{2}\log Z_{h}}{\partial h^{2}}(h_{0}) & =\frac{\partial}{\partial h}|_{h=h_{0}}\sum_{x\in\{\pm1\}^{n}}\frac{1}{Z_{h_{0}}}\left(\exp\{\sum_{i,j}J_{i,j}x_{i}x_{j}+\sum_{i}(h_{i}+h_{0})x_{i}\}\right)\left(\sum_{i}x_{i}\right)\\
 & =\boldsymbol{E}_{h_{0}}[(\sum_{i}x_{i})^{2}]-(\boldsymbol{E}_{h_{0}}[\sum_{i}x_{i}])^{2}\\
 & \geq0
\end{align*}
we see that $\log Z$ is convex in $h$; in particular, for any $h_{0}\in\R$
and any $\delta>0$, we have 

\[
\frac{\log Z(h_{0})-\log Z(h_{0}-\delta)}{\delta}\leq\frac{\partial\log Z}{\partial h}(h_{0})\leq\frac{\log Z(h_{0}+\delta)-\log Z(h_{0})}{\delta} 
\]
Finally, 
\begin{itemize}
\item By the mean value theorem, the LHS /RHS  of the equation above are given by 
$\boldsymbol{E}_{h'}[\sum_{i}x_{i}]$ and $\boldsymbol{E}_{h''}[\sum_{i}x_{i}]$, where
$h_0 - \delta < h' < h_0 < h'' < h_0 + \delta$.
\item By taking $\delta = \nu$ and using the oracle to compute the free energies within additive error $\epsilon \nu$, we can evaluate the LHS and RHS up to the desired error.
\end{itemize}

\end{proof}

We remark that: 
\begin{itemize}
\item Unfortunately, it is impossible to approximate in constant time the magnetization at the specified value of the external fields. For example, consider an Ising model on $4 n$ vertices, where 
$J_{i,j} = C$ for some large $C$ if $i,j \leq 2n$ and $J_{i,j} = 0$ otherwise. 
Let $h_i = 1$ if $i \in [2n+1,3n]$ and $h_i = -1$ if $i \in [3n+1,4n]$.
We set all the other $h_i$ to $0$ except that we set $h_I = X$, where 
$I$ is uniformly chosen in $[1,2n]$ and $X$ is uniformly chosen in $\{0,\pm 1\}$. Note that this is a dense Ising model as per our definition. Note also that on the nodes $[1,2n]$ we have the Ising model on the complete graph with one (random) node having external field.

It is easy to see that if $X = 0$, the magnetization is $0$. 
The fact that $C$ is a large constant implies that conditioning on one vertex taking the value $\pm$ results in a dramatic change in magnetization on the vertices $[1,2n]$. In particular, the magnetization is of order $n$ if $X = +1$ and is of order $-n$ if $X = -1$. 
It thus follows that we need $\Omega(n)$ queries in order to determine the magnetization in this case. 
We note that this example corresponds to a phase transition -- in particular, for every $\epsilon > 0$, if 
$h' > \epsilon$ then $\boldsymbol{E}_{h'}[\sum_{i}x_{i}] = \Omega(n)$ for all values of $X$ and $I$. 
See (\cite{ellis2007entropy}) for general references for the Ising model on the complete graph.

\item The results for computing the magnetization readily extend to other models. 
For example, for Potts models, we can compute for each color the expected number of nodes of that color 
(up to error $\epsilon \| \vec{J} \|_1$ and for an $\epsilon$ close external field). 
Similarly, it is easy to check we can compute other statistics at this accuracy. For instance, for the Ising model, we can approximate $\boldsymbol{E}[\sum a_i x_i]$ if $n \eta \|a\|_{\infty} \leq \| a \|_1$ for some $\eta > 0$.

\end{itemize}

\section{Appendix: Sample complexity lower bound}
In this section, we will provide a lower bound on the number of vertices which need to be sampled in order to provide an approximation of the quality guaranteed by \cref{thm:sample-complexity-free-energy}. We will find it convenient to make the following definition.
\begin{defn}
An Ising model is \emph{$\Delta$-dense} if $\Delta \|\vec{J}\|_{\infty} \leq \frac{\|\vec{J}\|_1}{n^2}$.
\end{defn}

For the rest of this section, we will focus on $\Delta$-dense ferromagnetic Ising models for which $n^2 \leq \|\vec{J}\|_{1} \leq n^3$. 
Note that for such Ising models, 
$$2000\epsilon n \left(\|J\|_F + \epsilon n \|\vec{J}\|_{\infty} + (\epsilon^3 n)^{-1/3}\|J\|_{F}^{2/3}\log^{1/3}(n\|J\|_F + e) + 1\right) \leq 5000\frac{\epsilon}{\sqrt{\Delta}}\|\vec{J}\|_{1},$$
provided that $n^{-1/4} \leq \epsilon \leq \sqrt{\Delta}$. 

\begin{theorem}\label{thm-qualitative-tightness}
Fix $\epsilon, \Delta \in (0,1/4)$. For any (possibly randomized) algorithm
$\mathcal A$ which probes at most $k := \frac{1}{8\epsilon \Delta}$ entries
of $J$ before returning an estimate to $\F$, there exists a $\Delta$-dense
input instance $J$ such that $\mathcal A$ makes error at least $\epsilon \|\vec J\|_1/4$ with probability at least $1/4$.
\end{theorem}

Before proving this theorem, let us show how it gives the desired sample complexity lower bound. 

\begin{proof}[Proof of \cref{thm:sample-complexity-lower-bound}] Let $\epsilon > 0$. Applying \cref{thm-qualitative-tightness} with $\Delta = 1/8$ and $C\epsilon$ shows that there exists a $\Delta$-dense instance $J$ such that any algorithm $\mathcal{A}$ which samples at most $1/C\epsilon$ entries of $J$ before returning an estimate to $\F$ makes an error of at least $C\epsilon \|\vec{J}\|_{1}/4$ with probability at least $1/4$. Since any algorithm which samples $q$ vertices from $[n]$ can probe at most $q^{2}$ entries of $J$, this applies, in particular, to any algorithm which samples at most $1/\sqrt{C\epsilon}$ vertices from $[n]$. Taking $C=60000$ gives the desired conclusion. 
\end{proof}
\begin{proof}[Proof of \cref{thm-qualitative-tightness}]
We prove the claim by reduction to a hypothesis testing problem. Specifically, we show that
there exist two different dense Ising models $J_M$ and $J'_M$ with free energies that are at least $\epsilon\|\vec J'_M\|_1/2$-far apart (where $\|\vec J_M\| > \|\vec J'_M\|$) such that no algorithm which makes only $k$ probes can distinguish between the two with probability greater than $3/4$. 
This immediately implies that for any algorithm $\mathcal A$ to estimate $\F$ and for at least one of the two inputs, $\mathcal A$ must make error at least $\epsilon\|\vec J'_M\|_1/4$ with probability at least $1/4$ when given this input ---
otherwise, we could use the output of $\mathcal A$ to distinguish the two models with probability better
than $3/4$, simply by checking which $\F$ the output is closer to.

Let $n$ be an instance size to be taken sufficiently large, and consider two $\Delta$-dense ferromagnetic Ising models defined
as follows: 
\begin{itemize}
\item $J_{M}$, for which the underlying graph is the complete graph on
$n$ vertices, $\epsilon\Delta {n \choose 2}$ many of the edges are randomly selected
to have weight $\frac{M}{\Delta}$, and the remaining $(1-\epsilon \Delta){n \choose 2}$ many
 edges are assigned weight $M$. Note that since $\|\vec{J}_{M}\|_{\infty}=\frac{M}{\Delta}$
and $\|\vec{J}_{M}\|_{1}=2(\epsilon \Delta {n \choose 2}\frac{M}{\Delta}+(1-\epsilon \Delta){n \choose 2}M) = 2(1 + \epsilon(1 - \Delta)) M {n \choose 2}$,
this model is indeed $\Delta$-dense for $n$ sufficiently large. 
\item $J'_{M}$, for which the underlying graph is the complete graph on
$n$ vertices and all edges have weight $M$. 
\end{itemize}
We denote the free energies of these models by $\F_{M}$ and
$\F'_{M}$ respectively. It is easily seen that $\lim_{M\rightarrow\infty}\frac{\F_{M}}{M}=\lim_{M\rightarrow\infty}\frac{\|\vec{J_{M}}\|_{1}}{M}=2(1 + \epsilon(1 - \Delta)){n \choose 2} \ge 2(1 + 3\epsilon/4){n \choose 2}$,
and that $\lim_{M\rightarrow\infty}\frac{\F'_{M}}{M}=\lim_{M\rightarrow\infty}\frac{\|\vec{J'_{M}}\|_{1}}{M}=2{n \choose 2}$.
Therefore, for $M$ sufficiently large, it follows that 
$|\F_M - \F'_M| \ge 
(\epsilon/2) \|\vec{J'_M}\|_1$.

Now, we show that no algorithm $\mathcal A$ can distinguish between
$J_M$ and $J'_M$ with probability greater than $3/4$ with only $k$ probes.
We fix a 50/50
split between $J_M$ and $J'_M$ on our input $J$ to algorithm $\mathcal A$.
Since the randomized algorithm $\mathcal A$ can be viewed as a mixture
over deterministic algorithms, 
there must exist a deterministic algorithm $\mathcal A'$ with success probability
in distinguishing $J_M$ from $J'_M$ at least as large as $\mathcal A$. Let $(u_1,v_1)$
be the first edge queried by $\mathcal A'$, let $(u_2,v_2)$ be the next edge queried
assuming $J_{u_1 v_1} = M$, and define $(u_3,v_3), \ldots, (u_k,v_k)$ similarly (without loss of generality,
the algorithm uses all $k$ of its available queries). Let $E$ be the event that
$J_{u_1,v_1}, \ldots, J_{u_k,v_k}$ are all equal to $M$. Event $E$ always
happens under $J_M$, and 
we see that $\Pr(E|J = J'_M) \ge 1 - k \frac{\epsilon \Delta n(n-1)/2}{n(n - 1)/2 - k} \ge 1 - 2k \epsilon \Delta$ for $n > 4k$. Thus, the total variation distance between the observed distribution
under $J_M$ and $J'_M$ is at most $2k \epsilon \Delta$,
so by the Neyman-Pearson lemma, we know $\mathcal A'$ fails with probability at least $(1/2)(1 - 2k \epsilon \Delta)$.
Therefore for $k \le \frac{1}{4\epsilon\Delta}$ we see that $\mathcal A'$ fails with probability at least $1/4$,
which proves the result.
\end{proof}

\end{document}